\documentclass{llncs}
\usepackage[T1]{fontenc}
\usepackage{graphicx}

\usepackage[english]{babel}
\usepackage{dsfont, amsmath}

\usepackage{caption}
\usepackage{subcaption}  % Para tener captions dentro de las subfiguras
\usepackage{amssymb}
\usepackage{booktabs} % For better table formatting
\usepackage{graphicx} % Required for inserting images
\usepackage[colorlinks=true, urlcolor=blue, linkcolor=blue, citecolor=green]{hyperref}
\usepackage{tikz}
\usetikzlibrary{arrows.meta}
\usetikzlibrary{decorations.pathmorphing}
\usepackage{algorithm}
\usepackage{algpseudocode}
\usepackage{pgf} % For arithmetic operations in TikZ
\usepackage{comment}
\usepackage{booktabs}
\usepackage{url} %For writing paths
\usepackage[normalem]{ulem}
\usepackage{multirow}
\usepackage{placeins} % For removing strange spaces

\usepackage[style=apa]{biblatex} % <- AGREGAR ESTA LÍNEA
\newcommand{\entities}{\texttt{ent}}

\newcommand{\sharpPhard}{\textsc{\#P-hard}}

\newcommand{\R}{\mathds{R}}
\newcommand{\perm}{perm}

\newcommand{\SHAPscore}{\textsc{SHAP-score}}

\newcommand{\charactheristicFunction}{\nu}
\newcommand{\players}{\mathcal{I}}

\newcommand{\drawUnrelatedTreeWithColor}[5]{
    \node[circle, draw=#5] (#1) at (#2, #3) {#4};
    \pgfmathsetmacro{\x}{#2-2}
    \pgfmathsetmacro{\y}{#3-0.3} 
    \draw[#5, wiggly] (#2, \y)
        -- ++(-1,-1.6) 
        -- ++(2,0) 
        -- cycle;
    \node[text=#5] at (\x, \y) {#4 subtree};
}

\newcommand{\drawUnrelatedTree}[4]{
    \drawUnrelatedTreeWithColor{#1}{#2}{#3}{#4}{red}
}

\usepackage[backgroundcolor=orange, textcolor=black, textsize=tiny]{todonotes}

\newcommand{\topo}{topos}

\newcommand{\childNetwork}{\emph{Child}}
\newcommand{\cancerNetwork}{\emph{Cancer}}

\tikzset{
	nodo/.style={
		shape=circle,
		draw=black,
		line width=1pt,   
		minimum size=7mm
	},
	arista/.style={
		line width=1pt,  
		-{Latex[length=3mm]}  
	},  
	mySnake/.style={
		decorate, 
		decoration={snake, amplitude=.4mm, segment length=4mm, post length=1mm}
	}
}

\begin{document}
%
% ENGLISH VERSION
\title{Beyond Shapley: Efficient Computation of Asymmetric Shapley Values}

\author{Ezequiel Companeetz\inst{1},
Santiago Cifuentes\inst{2},
Sergio Abriola\inst{2}}
\authorrunning{Companeetz et al.}
\institute{Departamento de Computación, Facultad de Ciencias Exactas y Naturales, UBA \and
Instituto de Ciencias de la Computación (ICC), CONICET}
\maketitle              % typeset the header of the contribution
\begin{abstract}
We address the problem of explainability in machine learning models through feature attribution methods. In particular, we consider a variant of Shapley values known as Asymmetric Shapley Values (ASV), which enables the incorporation of causal knowledge into model-agnostic explanations through the use of a causal graph. We show that in certain contexts in which the computation of SHAP is \#\textsc{P}-hard, the exact computation of ASV can be done in polynomial time.
To extend this algorithmic result, we introduce a notion of equivalence classes over the topological orderings of the underlying causal graph, which is useful to reduce the time to compute ASV. In particular, we present a polynomial-time algorithm (in the number of equivalence classes) to compute it whenever the causal graph is a rooted directed tree.
Finally, we develop an algorithm for approximating ASV in arbitrary causal DAGs which relies on a procedure to sample topological orderings uniformly at random. To implement this sampling mechanism we leverage known algorithms as well as simpler alternatives. Our experimental results demonstrate the practical viability of the proposed approach in realistic causal structures.

% keywords in lowercase
\end{abstract}

% SPANISH VERSION
% Clear the previous title info and set Spanish version

% Spanish title and abstract
\begin{center}
\vspace{1cm}
{\Large \bfseries\boldmath Más allá de Shapley: Algoritmos Eficientes para calcular los Shapley Values Asimétricos}

\vspace{0.5cm}
\end{center}

\begin{abstract}
\noindent Abordamos el problema de la explicabilidad en modelos de aprendizaje automático mediante métodos de \textit{feature attribution}. En particular, consideramos una variante de los \textit{Shapley values} conocida como \textit{Asymmetric Shapley Values} (ASV), que permite incorporar conocimiento causal en explicaciones \textit{model-agnostic} mediante el uso de un grafo causal. Mostramos que, en ciertos contextos en los que el cálculo de SHAP es \#\textsc{P}-hard, el cálculo exacto de ASV puede realizarse polinomial.
Para extender este resultado algorítmico, introducimos una noción de clases de equivalencia sobre los ordenamientos topológicos del grafo causal subyacente, lo cual resulta útil para reducir el tiempo de cómputo de ASV. En particular, presentamos un algoritmo en tiempo polinomial (en el número de clases de equivalencia) para calcularlo cuando el grafo causal es un árbol dirigido enraizado.
Finalmente, desarrollamos un algoritmo para aproximar ASV en DAGs causales arbitrarios, basado en un procedimiento que muestrea ordenamientos topológicos uniformemente al azar. Para implementar este mecanismo de muestreo consideramos algoritmos conocidos así como alternativas más simples. Nuestros resultados experimentales demuestran la viabilidad práctica del enfoque propuesto en estructuras causales realistas.
\bigskip

\end{abstract}

\vspace{0.5cm}
\setcounter{page}{1}

\vspace{-1cm}

\section{Introduction}

% In recent years, the capabilities of artificial intelligence models have grown rapidly, along with the complexity of their architectures. 
%This progress has enabled the solution of tasks that once appeared unattainable, such as the use of language models to solve complex mathematical problems \parencite{frontierMath}. 
%At the same time, current systems have reached scales far beyond those of a few years ago, both in terms of the number of parameters and structural complexity \parencite{EpochLargeScaleModels2024}. 
%This lack of transparency is especially problematic in sensitive settings, such as medical diagnosis or high-stakes decision-making, where producing an accurate prediction is not enough: it is equally important to understand how that prediction was made. In response to this challenge, the field of Explainable Artificial Intelligence (XAI) has emerged \parencite{mersha2024explainable}, aiming to develop methods capable of making model decisions more interpretable.

The field of XAI (\textit{Explainable Artificial Intelligence}) aims to develop methods capable of making model decisions more interpretable \parencite{mersha2024explainable}. Within this field, one of the most influential approaches is feature attribution, where each input variable is assigned a score reflecting its contribution to the prediction. 
In this line of work, SHAP is one of the most well-known methods \parencite{shapOriginalPaper}. It is based on the Shapley values, a concept originating in cooperative game theory \parencite{shapley1953value}. 
Intuitively, Shapley values provide a fair way to distribute the total value achieved by a coalition among its participants, according to their average marginal contribution when joining. 
Translated to machine learning, features play the role of players, and the value assigned to each feature aims to quantify its contribution to the model’s prediction for a given instance.

Despite its popularity, SHAP has important limitations. On the one hand, depending on the family of ML models under consideration and the underlying distribution of the data, the computation of the Shapley values can be tractable or not. For example,
it is well-known that for product distributions the complexity of computing these scores is equivalent to the complexity of computing the average 
value of the model \parencite{van2022tractability}. In particular, this implies that there exist efficient algorithms for the case where the models are given, for example, by decision trees or by restricted families of circuits \parencite{arenas2023complexity}, while the problem is intractable whenever the models are more expressive, as in the case of neural nets. 
For more general families of distributions the problem becomes much harder: for instance, when considering Naive Bayes distributions, the computation of this score is intractable already when considering trivial models that only depend on one feature \parencite{van2022tractability}.

On the other hand, there are also concerns about the quality of the explanations it produces. In particular, it is not always clear how to interpret, in the context of AI models, the axioms inherited from game theory, and several works point out that these explanations may fail to adequately capture relationships 
such as dependence, relevance, or causality between variables \parencite{fryer2021shapley, marques2023logic, huang2023inadequacy}. More generally, there is no absolute consensus on what should be considered an “explanation” in AI: many current techniques are pragmatic approximations to a much broader problem related to understanding, causality, and human interpretability \parencite{MILLER20191, lipton2017mythosmodelinterpretability}.

The Asymmetric Shapley Values (ASV from now on) are a proposal introduced by \cite{frye2019asymmetric} to overcome some of these limitations. This variant incorporates causal information into the computation of attribution scores by restricting attention to permutations of variables that respect a causal graph. 
In this way, instead of treating all features symmetrically, priority is given to the more fine-grained ones (in the sense that the information they provide is not present in the other variables). This can produce explanations that are more faithful to the data-generating process and can help better detect phenomena such as bias or discrimination. 
Moreover, although ASV was primarily proposed to improve explanatory quality, it naturally raises the question of whether this additional restriction may also provide computational benefits.

Therefore, in this paper we address the problem of efficiently computing the ASV scores, which was not explored thoroughly in previous literature. We will prove that for some families of models for which
computing the Shapley values is intractable there are still polynomial time algorithms for computing the ASV. Inspired by this initial result, we develop an algorithm to compute the ASV exactly in more general contexts
by exploiting the fact that many permutations of the features can be grouped in a unique equivalence class to simplify computations. Finally, we also propose a flexible approximation scheme that allows to approximate the ASV as long as it is possible to (1) Sample a topological order of the causal graph uniformly at random, and (2) Sample a data point with respect to the data distribution after conditioning a subset of the features to have specific values. We instantiate all these general algorithms with particular 
examples of causal graphs and machine learning models and show experimentally that they are efficient.

\section{Definitions}

% En esta sección introducimos las nociones formales que utilizaremos a lo largo del trabajo. Primero definimos los \textit{Shapley values} y su instanciación al problema de \textit{feature attribution}. Luego presentamos los \textit{Asymmetric Shapley Values} (ASV), que incorporan relaciones entre las variables a la atribución mediante un grafo causal. Finalmente, introducimos las redes bayesianas como modelo para la distribución de los datos y describimos las familias restringidas de redes que consideraremos en nuestros resultados.

%\subsection{Shapley values y feature attribution}

Let $X$ be a finite set of features. A (binary) \textit{entity} over $X$ is a mapping $e:X \to \{0, 1\}$ assigning a value to every feature from $X$. We denote by $\entities{}(X)$ the set of all entities, and we will assume that some probability distribution is provided over this set.
A (binary) classifier is a mapping $M: \entities{}(X) \to \{0, 1\}$. We restrict to both binary entities and classifiers only for simplicity: all the results and algorithms we obtain can be extended to more general cases straightforwardly. 

Given a model $M$ and an entity $e$, a \textit{feature attribution score} is a mapping $\phi_{M,e}:X \to \R$. Intuitively, the value $\phi_{M,e}(x)$ indicates the relevance of feature $x$ with respect to prediction $M(e)$. As mentioned before, one of the most 
studied feature attribution schemes is the \SHAPscore{} \parencite{shapOriginalPaper}, which is theoretically grounded on the \emph{Shapley values} \parencite{shapley1953value} from cooperative game theory. In that context, the Shapley values
come up as the ``most fair'' way to distribute some capital earned during a team game. More precisely, given a finite set of players $\players$ and a \emph{characteristic function} $\charactheristicFunction:\mathcal{P}(\players)\to\mathbb{R}$ describing, for each subset
of the players, the reward that such coalition would obtain in some fixed game; the shapley values are the unique assignment $\varphi_\nu: \players \to \R$ of reward to each player $i$ satisfying the following desired properties:

\begin{enumerate}
	\item Efficiency: All the reward is distributed, i.e. $\sum_{i \in \players} \varphi_{\nu}(i) = \nu(\players) - \nu(\emptyset)$.
	\item Symmetry: If $\nu(S \cup \{i\}) = \nu(S \cup \{j\})$ for every $S \subseteq \players \setminus \{i,j\}$, then $\varphi_{\nu}(i) = \varphi_{\nu}(j)$.
	\item Null Player: If $\nu(S \cup \{i\}) = \nu(S)$ for every $S \subseteq \players$, then $\varphi_{\nu}(i) = 0$.
	\item Linearity: For any pair of characteristic functions $\nu_1,\nu_2$ and $a \in \R$ it holds that $\varphi_{a\nu_1 + \nu_2}(i) = a\varphi_{\nu_1}(i) + \varphi_{\nu_2}(i)$\footnote{The characteristic function $a \nu_1 + \nu_2$ is naturally defined as $(a \nu_1 + \nu_2)(S) = a \nu_1(S) + \nu_2(S)$.}. 
\end{enumerate}

Moreover, the Shapley values have an exact closed form: if we refer to players with indices in the range $\{1,\ldots, n\}$, we can write
\begin{align}\label{eq:shapley}
	\varphi_{\nu}(i) =  \frac{1}{|\players|!} \sum_{\pi \in perm(\players)} \left[ \nu(\pi_{<i} \cup \{i\}) - \nu(\pi_{<i}) \right]
\end{align}
where $perm(\players)$ denotes the set of all permutations of the number from 1 to $n$ and $\pi_{<i}$ denotes the set of indices before $i$ in the permutation $\pi$ (i.e. $\pi_{<i} = \{j : \pi(j) < \pi(i)\}$). Fundamentally, Eq.~\eqref{eq:shapley} is summing across all possible ways
of ordering the players, considering the impact of the inclusion of player $i$ to the reward of the game when adding the players according to each of these orderings.

In the context of machine learning explainability, the Shapley values are instantiated by thinking of the features as players and considering the characteristic function
\begin{align}\label{eq:char_fun_for_ML}
	\nu_{M,e}(S) = \mathbb{E}[M(e') | S, e] = \sum_{e'\in \entities(X)} M(e') \Pr[e'| \{e'(x) = e(x) : x \in S\}] 
\end{align}
which represents the expected value of the model $M$ when restricting the features from $S$ to have the value dictated by entity $e$.
Intuitively, if some prediction $M(e) = 1$ depends strongly on the fact that some feature $x$ has value $e(x) = 1$, then it will hold that $\nu_{M,e}(S)$ is close to 1 if $x \in S$ and far from 1 otherwise. For example, if the model $M$ merely outputs the value of $x$ (i.e. $M(e) = e(x)$) and 
we assume that the value of $x$ is independent of the other features and uniformly distributed between $0$ and $1$, then $\nu_{M,e}(S) = 1$ if $x\in S$ while $\nu_{M,e}(S) = \frac{1}{2}$ otherwise. Thus, it is expected that $\nu_{M,e}$ will give a bigger value to those features more relevant for the prediction.

As mentioned in the introduction, the Shapley values may fail to distinguish features that are heavily correlated. For example, if some feature $j$ depends deterministically of some other feature $i$ it can happen that both features receive the same score, even though $i$ is intuitively more relevant
for \textit{any} model and prediction. Note that the \textit{Symmetry} property essentially implies that this should always happen whenever two features share the same information. Thus, to avoid this problem it is somewhat necessary to relax that axiom and define a different score.
This is the idea underlying the \textit{Asymmetric Shapley Values}.

\begin{definition}[Asymmetric Shapley Values \parencite{frye2019asymmetric}]\label{def:asv}
	Given a model $M$, an entity $e$, a feature $x$ and some weight function $w: perm(X) \to \R_{\geq 0}$ satisfying $\sum_{\pi \in perm(X)} w(\pi) = 1$, we define the Asymmetric Shapley Value of $x$ as
	\begin{align*}
		\phi^{ASV}_{M,e}(x) = \sum_{\pi \in perm(X)} w(\pi) \left[ \nu_{M,e}(\pi_{<i} \cup \{i\}) - \nu_{M,e}(\pi_{<i}) \right]
	\end{align*}
	where $\nu_{M,e}$ is defined in Eq.~\eqref{eq:char_fun_for_ML} 
\end{definition}

This definition is derived from the corresponding notion of asymmetric shapley values from cooperative game theory. The function $w$ is useful for weighting some orderings above others. For example, if some feature $j$ is a function of $i$ it is possible to introduce this fact into the score by assigning weight zero to all orderings in which $i$
comes after $j$. Note that the Shapley values are obtained by taking $w(\pi) = \frac{1}{|X|!}$, and we will refer to this particular case as $\phi^{Shap}_{M,e}$.

\cite{frye2019asymmetric} proposes to define the weight function $w$ by considering a \textit{causal DAG} $C$ over the features. More precisely, $C$ is a directed acyclic graph whose nodes are the set of features $X$ and where an edge directed from $i$ to $j$ indicates that $j$ depends on $i$. Then, they define the weight function
\begin{align}\label{eq:w_in_ASV}
	w_C(\pi) = \frac{\mathbb{I}_{\{\pi \text{ is a toposort}\}}}{|topos(C)|} 	
\end{align}
where $topos(C)$ is the set of all topological orderings of $C$. This weight function generalizes the idea mentioned earlier of considering only those orderings consistent with the known causality of the features. In the rest of this paper we will be interested in developing efficient algorithms for computing the ASV from Def.~\ref{def:asv} whenever the weight function is given by 
Eq.~\eqref{eq:w_in_ASV}. 

Moreover, we will consider that the distribution of the space of entities is described by a Bayesian Network $N$. We refer to the reader to standard references for their precise formal definitions \parencite{pearl1986bayesianInference} while here we describe them conceptually: a bayesian network $N$
defines a distribution for $\entities(X)$ with a DAG whose nodes are the features from $X$. Every node $x$ has a set of parents $p(x) \subseteq X$, and the network provides, for each way of assigning the features in $p(x)$, the conditional distribution of $x$. See Figure \ref{fig:cancer_bn_with_cpts} for an example. Given $N$ and any entity $e$ it is possible to compute
the probability of $e$ by traversing the DAG of $N$ in any topological sort multiplying the corresponding conditional probabilities.

One of the simplest types of Bayesian Networks is the Naive Bayes, which was already mentioned in regard to intractability results for the Shapley values. This network has a particular structure, in which a single node is parent to all the others and no other pairs of nodes are connected. See Figure \ref{fig:naive_bayes_network} for an example.

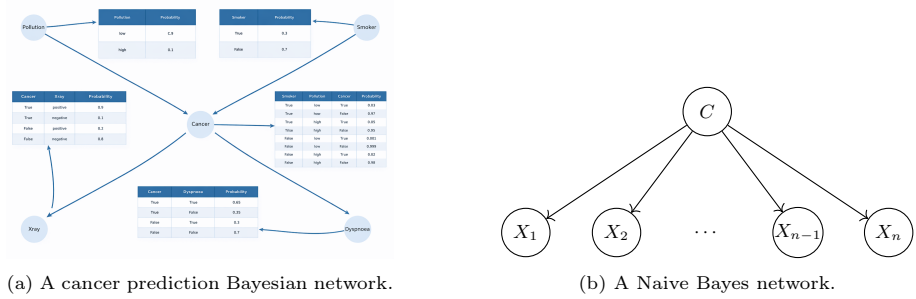
\begin{figure}[ht]
	\centering
	
	\begin{subfigure}[b]{0.42\textwidth}
		\centering
		\includegraphics[width=\linewidth]{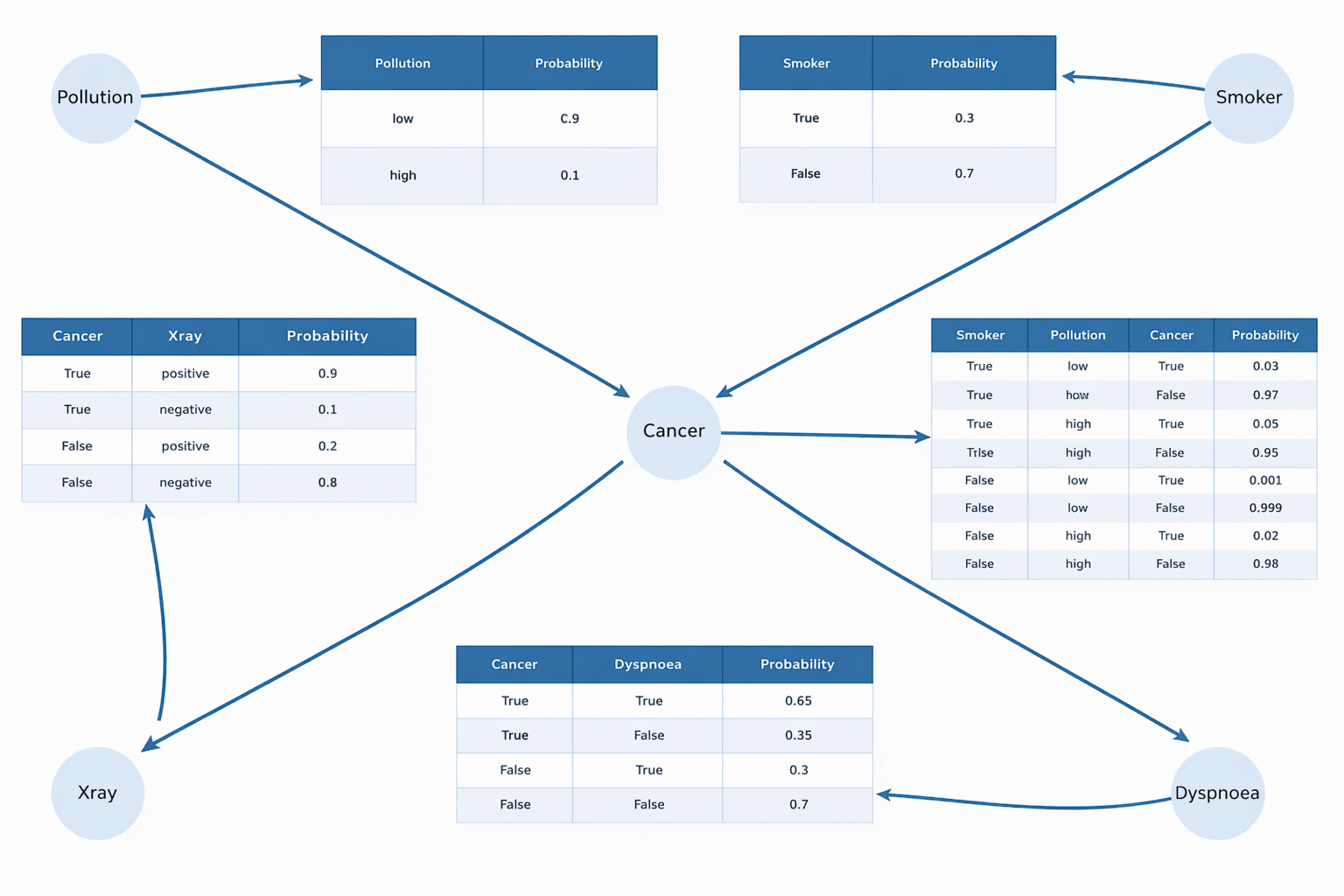}
		\caption{A cancer prediction Bayesian network.}
		\label{fig:cancer_bn_with_cpts}
	\end{subfigure}
	\hfill
	\begin{subfigure}[b]{0.48\textwidth}
		\centering
		\begin{tikzpicture}[scale=0.8, transform shape,
			nodo/.style={circle, draw, minimum size=8mm, inner sep=0pt}
			]
			% Root
			\node[nodo] (C) at (0,2) {$C$};
			
			% Features
			\node[nodo] (X1) at (-3,0) {$X_1$};
			\node[nodo] (X2) at (-1.5,0) {$X_2$};
			\node[draw=none] (dots) at (0,0) {$\cdots$};
			\node[nodo] (Xn1) at (1.5,0) {$X_{n-1}$};
			\node[nodo] (Xn) at (3,0) {$X_n$};
			
			% Edges
			\draw[->] (C) -- (X1);
			\draw[->] (C) -- (X2);
			\draw[->] (C) -- (Xn1);
			\draw[->] (C) -- (Xn);
		\end{tikzpicture}
		\caption{A Naive Bayes network.}
		\label{fig:naive_bayes_network}
	\end{subfigure}
	
	\caption{Examples of Bayesian network structures.}
	\label{fig:bayesian_network_examples}
\end{figure}
\enlargethispage{2\baselineskip} % This is done so that the footnote can fit.

\vspace{-1cm}

A key point for our purposes is that computing the characteristic function (from Eq.~\eqref{eq:char_fun_for_ML}) under a Bayesian network distribution requires probabilistic inference. Since exact inference in Bayesian networks is $\sharpPhard$ \footnote{\#\textsc{P} is a class of \textit{counting} problems. The \#\textsc{P}-hard problems from \#\textsc{P} are the hardest problems in this class, and it is known that if they can be solved in polynomial time then $\textsc{P} = \textsc{NP}$. Then, under the usual theoretical assumption, \#\textsc{P}-hard problems cannot be solved efficiently.} in general \parencite{Cooper1990}, we focus on restricted networks such as polytrees, where inference can be performed in polynomial time.

Sometimes we will identify the network itself as a causal graph. Thus, the network will provide both the distribution necessary to compute the characteristic function $\nu_{M,e}$ and the causal graph necessary to define the weight function $w^C$.
This identification is not always appropriate. In general, a Bayesian network represents a joint probability distribution and its conditional independencies, but its directed edges do not necessarily encode causal relations. Therefore, using the same graph as both a probabilistic model and a causal graph is only justified when the network is assumed to reflect the underlying causal structure. Otherwise, this should be understood merely as a modeling simplification.

\section{Theoretical Results}

\subsection{Tractability of ASV}

As mentioned before, the computation of the Shapley Values is intractable (in particular, $\sharpPhard$) even for trivial models
if the underlying distribution is represented by a Naive Bayes network. We show that this is not the case for the ASV whenever we use the bayesian network as the causal graph\footnote{Due to space limitations, we omit all proofs and defer the complete arguments to the supplementary material.}.

\begin{theorem}\label{the:asv_naive_bayes_tractable_clean}
	Let $\mathcal{F}$ be any family of models. Then, the ASV can be computed in polynomial time for the family $\mathcal{F}$ under Naive Bayes distributions (where the causal graph is the network itself) if and only if the Shapley values can be computed in polynomial time for the family $\mathcal{F}$ under product distributions\footnote{With our notation, a product distribution over $\entities{}(X)$ is a distribution that factorizes as $\Pr\left[e\right] = \prod_{x \in X}\Pr\left[ x = e(x) \right]$.}.
\end{theorem}

Theorem~\ref{the:asv_naive_bayes_tractable_clean} shows that, in contrast to the Shapley values, the ASV can be computed under Naive Bayes distributions for simple models such as decision trees. Intuitively, the reason for this is that the causality imposed by the distribution is mimicked by the way that only certain permutations are considered,
 and thus many symmetries can be exploited. We believe this motivates the idea of studying the complexity of computing the ASV whenever the distribution is also used as the causal graph. 

\vspace{-0.3cm}

\subsection{Equivalence classes of topological sorts}

The Shapley values can also be written in closed form as
\begin{align*}
	\phi^{Shap}_{M,e}(x) = \sum_{S \subseteq X \setminus \{x\}} \frac{|S|!(|X|-|S|-1)!}{|X|!}\left[\nu_{M,e}(S \cup \{x\}) - \nu_{M,e}(S)\right] 
\end{align*}
This rewriting is obtained by observing that many permutations give rise to the same set $\pi_{<x}$: more precisely, if $|\pi_{<x}| = s$ there are $s!(|X|-s-1)!$ permutations $\pi'$ satisfying $\pi'_{<x} = \pi_{<x}$. Computing the Shapley values using this formula reduces the number of times that $\nu_{M,e}$ has to be computed, although the number of calls remains
exponential on the number of features.

We will follow a similar strategy to compute the ASV. We define a notion of equivalence between permutations.

\begin{definition}[Equivalence relation for permutations]\label{def:equiv_relation_left_right}
	Let $C$ be a causal graph for features $X$ and fix some feature $x \in X$. We define an equivalence relation $\sim_x$ over $\topo(C)$ as
	\[
	\pi^1 \sim_x \pi^2
	\iff
	\pi^1_{<x}=\pi^2_{<x}.
	\]
\end{definition}

\newcommand{\equivalenceClasses}{\Lambda_C^x}

Let $\equivalenceClasses$ be the set of equivalence classes. It is immediate from the definition that
\begin{align}\label{eq:equivalence_classes_into_shap}
	\phi^{ASV}_{M,e} (x) = \frac{1}{|topos(C)|} \sum_{[\pi] \in \equivalenceClasses} |[\pi]| \left[ \nu_{M,e}(\pi_{<x} \cup \{x\}) - \nu_{M,e}(\pi_{<x}) \right]
\end{align}
Therefore, instead of invoking $\nu_{M,e}$ $2|topos(C)|$ times it is enough to compute it $2|\equivalenceClasses|$ times. This value can be at most $2^{|X|}$, but the hope is that in practical scenarios it is smaller. From now on, we will restrict to the case in which causal graphs are given 
by \textit{rooted directed trees}, polytrees with nodes that have one or fewer parents. In this case, note that we can split the nodes in three groups with respect to $x$: those that are descendants of $x$, those that are ancestors of $x$, and the \textit{unrelated} ones. Observe that descendant/ancestors are never/always in the set $\pi_{<x}$, for any $x$. Thus, the structure of the equivalence classes depends only on the unrelated nodes. See Figure~\ref{fig:ASV_forest_example} for a visual explanation.

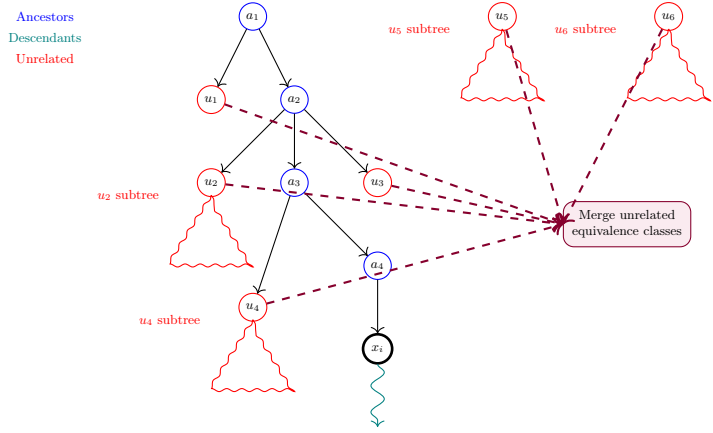
\begin{figure}[H]
	\centering
	\begin{tikzpicture}[scale=.55, transform shape, 
		unrelated/.style={circle, draw=red},
		ancestor/.style={circle, draw=blue},
		mergebox/.style={draw=purple!70!black, rounded corners, fill=purple!8, align=center, inner sep=6pt},
		wiggly/.style={decorate, decoration={snake, amplitude=.2mm, segment length=2mm}}
		]
		\node[text=blue] at (-5, 0) {Ancestors};
		\node[text=teal] at (-5, -0.5) {Descendants};
		\node[text=red] at (-5, -1) {Unrelated};        
		
		\node[ancestor] (a1) at (0, 0) {$a_1$};
		\node[unrelated] (u1) at (-1, -2) {$u_1$};
		\node[ancestor] (a2) at (1, -2) {$a_2$};
		
		\drawUnrelatedTree{u2}{-1}{-4}{$u_2$}
		\node[ancestor] (a3) at (1, -4) {$a_3$};
		\node[unrelated] (u3) at (3, -4) {$u_3$};
		
		\drawUnrelatedTree{u4}{0}{-7}{$u_4$}
		\node[ancestor] (a4) at (3, -6) {$a_4$};
		
		\node[nodo] (xi) at (3, -8) {$x_i$};
		
		\drawUnrelatedTree{r1}{6}{0}{$u_5$} 
		\drawUnrelatedTree{r2}{10}{0}{$u_6$} 
		
		\node[draw=none, fill=none] (hi) at (3, -10) {};
		
		\path [->] (a1) edge (u1);
		\path [->] (a1) edge (a2);
		\path [->] (a2) edge (u2);
		\path [->] (a2) edge (a3);
		\path [->] (a2) edge (u3);
		\path [->] (a3) edge (u4);
		\path [->] (a3) edge (a4);
		\path [->] (a4) edge (xi);      
		\path [->, teal] (xi) edge[decorate, decoration={snake}] (hi);
		
		% Merge box
		\node[mergebox] (merge) at (9,-5) {
			Merge unrelated\\
			equivalence classes
		};
		
		% Dashed arrows from unrelated roots
		\draw[->, dashed, purple!70!black, thick] (u1) -- (merge.west);
		\draw[->, dashed, purple!70!black, thick] (u2) -- (merge.west);
		\draw[->, dashed, purple!70!black, thick] (u3) -- (merge.west);
		\draw[->, dashed, purple!70!black, thick] (u4) -- (merge.west);
		\draw[->, dashed, purple!70!black, thick] (r1) -- (merge.west);
		\draw[->, dashed, purple!70!black, thick] (r2) -- (merge.west);
		
	\end{tikzpicture}
	\caption{Graph partition into ancestors, descendants and unrelated nodes with respect to $x_i$, the feature for which we want to calculate the ASV. The unrelated rooted subtrees can be processed independently and then merged to construct the full equivalence classes.}
	\label{fig:ASV_forest_example}
\end{figure}

\vspace{-0.5cm}

We have a general bound on the size of $\equivalenceClasses$ depending only on the causal graph $C$:

\begin{proposition}\label{prop:bound}
	Let $C$ be a causal graph shaped as a rooted directed tree with $l$ leaves and depth $h$. Then, it holds that $|\equivalenceClasses|\leq h^l + 1$	
\end{proposition}

Proposition~\ref{prop:bound} indicates that the number of equivalence classes will be small when the tree is shallow. From the point of view of the causal graph, this corresponds to the case in which there aren't long causal dependencies between the features. In particular, there will be many features that are mutually independent.

For the case of directed trees we further develop an algorithm able to compute the set $\equivalenceClasses$ as well as the number of permutations inside each class (which is required to compute Eq.~\eqref{eq:equivalence_classes_into_shap}). The complexity of the algorithm
is polynomial on $|\equivalenceClasses|$.

\begin{proposition}\label{prop:algo}
	There is an algorithm able to compute $|[\pi]|$ for each $[\pi] \in \Lambda^x_C$ given feature $x$ and the causal graph $C$ (shaped as a rooted directed tree) on time polynomial on $|C|$ and $|\equivalenceClasses|$.
\end{proposition}

We note that once the set $\equivalenceClasses$ is computed it is possible to find any value $\phi_{M,e}^{ASV}(x)$ in time linear on $|\equivalenceClasses|$. Thus, even though this step can be costly
(because the bound from Proposition~\ref{prop:bound} implies that the algorithm may take exponential time on the size of the causal graph) is must be done only once. Moreover, it is independent of the model $M$, and thus it can be reused even
when the model is updated.

\vspace{-0.3cm}

\subsection{Approximating ASV through sampling}

\newcommand{\subsetPerm}{K}

A simple way to compute $\phi_{M,e}^{Shap}(x)$ consist on sampling a small subset $\subsetPerm \subseteq perm(X)$ at random and then approximating the score as
\begin{align*}
	\phi^{Shap}_{M,e}(x) \sim \frac{1}{|\subsetPerm|} \sum_{\pi \in \subsetPerm} \left[ \nu_{M,e}(\pi_{<x} \cup \{x\}) - \nu_{M,e}(\pi_{<x}) \right]
\end{align*}

For the ASV the analogous procedure consists on sampling topological orderings at random. In the next statement we formalize this strategy. Moreover, we observe that the algorithm is robust to error both at the level of the sampling procedure and the computation of $\nu_{M,e}$.

\newcommand{\procedureSampling}{\mathcal{P}_{samp}}
\newcommand{\procedureNu}{\mathcal{P}_{\nu}}

\begin{proposition}\label{prop:sampling_to_compute_ASV}
	Assume access to a procedure $\procedureSampling$ that samples topological orderings from a DAG with a distribution $\varepsilon_{samp}$-close to the uniform one\footnote{To measure distance between distributions we use the $1$-norm, i.e. given two distributions $\mathcal{D}_1$ and $\mathcal{D}_2$ over a finite domain $\Omega$ we define the distance between $\mathcal{D}_1$ and $\mathcal{D}_2$ as $||\mathcal{D}_1 - \mathcal{D}_2||_1 = \sum_{o \in \Omega} |\mathcal{D}_1(o) - \mathcal{D}_2(o)|$.},
	and to a procedure $\procedureNu$ that computes, given any subset of features $S \subseteq X$, a $\varepsilon_{\nu}$-approximation of the value $\nu_{M,e}(S)$. Then, it is possible to compute, with high probability, a $O(\varepsilon_{samp} + \varepsilon_{\nu})$-approximation of $\phi^{ASV}_{M,e}(x)$ using $O(1/\varepsilon_{\nu}^2)$ calls to $\procedureSampling$ and $\procedureNu$. 
\end{proposition}

\newcommand{\randomVar}{V^{M,e}}

Proposition~\ref{prop:sampling_to_compute_ASV} states that we only need to be able to sample the topological orderings \textit{almost} uniformly to compute the ASV. Also, it shows that if we approximate $\nu_{M,e}$ the overall \textit{additive}-error of the estimation
of $\phi^{ASV}_{M,e}(x)$ is also controlled. Note that implementing the procedure $\procedureNu$ is often easy: by definition $\nu_{M,e}(S)$ is the expected value of a random variable, and thus if we can sample from the set of entities after conditioning that $s = e(s)$ for every $s \in S$ we are done. This is the case for the product distribution, but also for any bayesian networks whose underlying graph is a polytree, because then the inference process can be implemented in polynomial time. 

Now we describe two ways to implement the procedure $\procedureSampling$. The first one can be understood by looking at Figure \ref{fig:topoSortSamplingExample}. The idea is that we can sample a topological ordering uniformly using a recursive procedure in which at each step we pick the left-most node of the ordering. In particular, in the beginning we must pick one of the sources of the causal graph (i.e. one of the nodes with in-degree 0). We should weight them according to the number of topological orderings that start with each respective node, and thus we effectively reduce the problem of sampling a topological ordering to the problem of counting the number of topological orderings. Even though this problem is intractable in general \parencite{countingLinearExtensions}, there are heuristic algorithms that can approximate this value or even exact algorithms for restricted cases \parencite{efficientCountingOfToposorts}. 

In particular, we develop a simple algorithm to count topological orders applicable to polytrees whose complexity depends on the maximum degree of a node. 
This algorithm computes the number of topological orderings of a polytree by performing a DFS over its underlying undirected graph and recursively combining the results of each node’s unvisited neighbors. Instead of only counting the number of orderings of each subtree, the algorithm keeps track of the positions in which a node may appear, together with how many topological orderings realize each position (this additional information is necessary to correctly handle nodes with multiple parents). At each step, the algorithm enumerates the valid relative orders of the current node and its unvisited neighbors, combines the recursive results of those neighbors, and merges equal positions. In this way, it computes all possible placements of the node and the corresponding number of topological orderings, yielding an exact counting algorithm for arbitrary polytrees.

\begin{figure}[ht]
	\centering
	\begin{tikzpicture}[scale=.6]
		% Define the rigth set of nodes
		\foreach \i in {1,2,3,4,5}
		\node[draw=none, circle, minimum size=5mm, inner sep=0pt] (L\i) at (4, -\i) {};
		
		% Define the left set of nodes
		\foreach \j in {1,2,3}
		\node[draw, circle, red, minimum size=5mm, inner sep=0pt] (source\j) at (0, -\j*2) {\scalebox{0.6}{$Source_\j$}};
		
		% Draw edges between nodes (example edges)
		\foreach \i in {1,2}
		\foreach \j in {1,2}
		\draw[->]  (source\j) -- (L\i); 
		
		\foreach \i in {4,5}
		\foreach \j in {2,3}
		\draw[->]  (source\j) -- (L\i); 
		
		\draw[->]  (source2) -- (L3);

		\draw [decorate, blue, decoration={random steps, segment length=10pt, amplitude=2pt}, thick]
		(4,-3) circle (2.4);
		
		%Number of órdenes topológicos for each node
		
		\node[draw=none,minimum size=3mm, inner sep=0pt] () at (0, -1) {\small \textcolor{orange}{$n_1$=100}};
		
		\node[draw=none,minimum size=3mm, inner sep=0pt] () at (0, -3) {\small \textcolor{orange}{$n_2$=300}};
		
		\node[draw=none,minimum size=3mm, inner sep=0pt] () at (0, -5) {\small \textcolor{orange}{$n_3$=200}};
	\end{tikzpicture}
	\caption{Illustration of the first step of the sampling algorithm. The candidate source nodes that may appear first in the ordering are shown in red. Each source node $s$ is annotated with the number of topological orderings $n_s$ in which it appears first. We will assign a probability to each $s$ as $n_s / topos(C)$.}
	\label{fig:topoSortSamplingExample}
\end{figure}
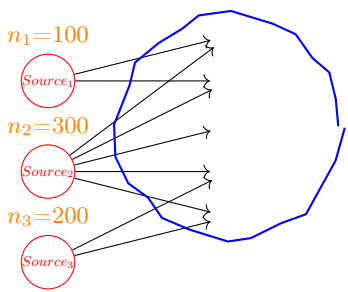

\vspace{-0.5cm}

Alternatively, we consider known algorithms from the literature able to approximate uniform sampling of topological orderings from arbitrary DAGs. We implement an algorithm based on random walks \parencite{bubley1999faster} with a polynomial (although of high degree) complexity. We note that there are more modern and asymptotically better algorithms \parencite{HUBER2006420}, but we believe that the hidden constants in them make them inefficient for the input sizes we consider.

\section{Experimental results}

We evaluate the proposed methods by comparing them against naive baselines in terms of computational cost and scalability. 
Experiments were conducted on two real Bayesian networks, \cancerNetwork{} and \childNetwork{}, obtained from the \texttt{bnlearn} repository. Both were used as data distributions and as causal DAGs. The \childNetwork{} network was modified into a polytree by removing edges that introduce cycles in the underlying undirected graph, where we computed the new conditional tables through marginalization.
Datasets were generated by sampling from these networks, and decision trees were trained accordingly. For ASV computation, the prediction variable was removed from the network to avoid trivial inference. To analyze the performance of the algorithms that compute the equivalence classes of topological sorts we also created artificial networks shaped as (1) Balanced binary trees, (2) Naive-bayes distributions, and (3) Trees composed by merging directed paths on their root (what we call \texttt{multiplePaths}).\footnote{All experiments were run on an Intel i5-7500 CPU @ 3.40GHz with 16 GB of RAM, using Ubuntu 22.04 and Python 3.12. The main libraries used were \texttt{pgmpy}, \texttt{sklearn}, \texttt{networkx}, and \texttt{shap}.}

\vspace{-3mm}

\paragraph{Equivalence Classes vs. Topological Orders} We compare the naive ASV computation based on iterating all topological orders with our approach based on equivalence classes. The key quantities are the sizes of $\topo(G)$ and $\Lambda^x_{G}$, as well as the cost of constructing each set.

\begin{figure}[ht]
	\centering
	\begin{subfigure}[b]{0.49\linewidth}
		\centering
		\includegraphics[width=\linewidth]{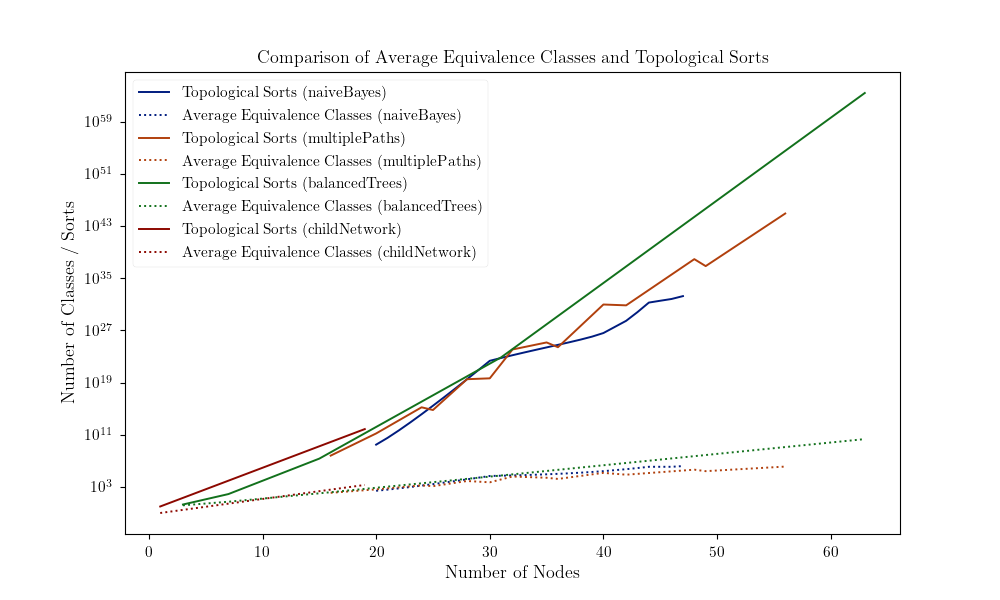}
		\caption{Number of equivalence classes vs. topological orders.}
		\label{fig:equivalenceClassesVsToposortsNumberPlot}
	\end{subfigure}
	\hfill
	\begin{subfigure}[b]{0.49\linewidth}
		\centering
		\includegraphics[width=\linewidth]{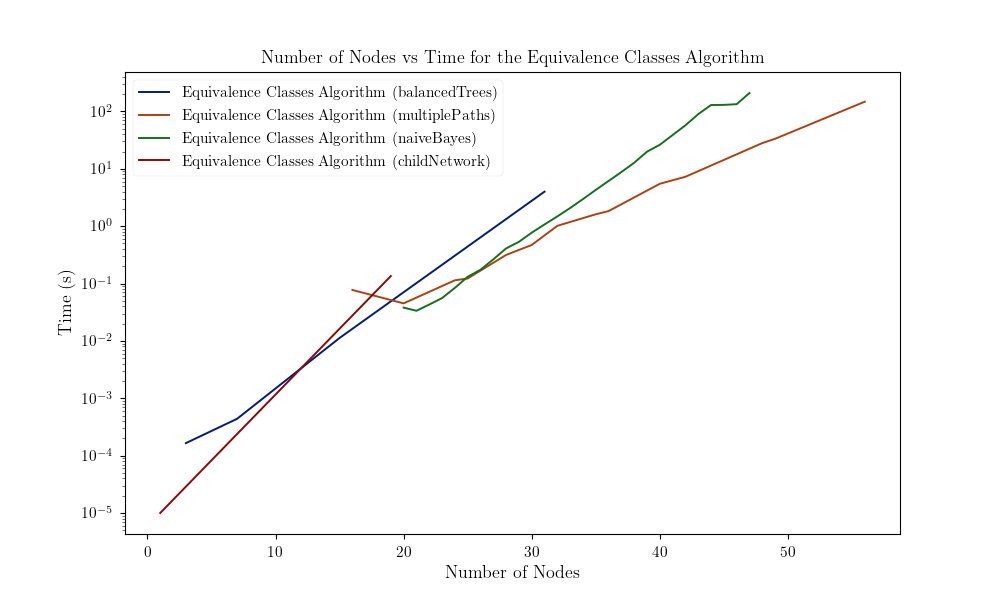}
		\caption{Runtime for computing equivalence classes.}
		\label{fig:equivalenceClassesTimePlot}
	\end{subfigure}
	\label{fig:eq_vs_topo_combined}
\end{figure}

Figure~\ref{fig:equivalenceClassesVsToposortsNumberPlot} shows that the number of equivalence classes grows significantly slower than the number of topological orders. For instance, in \childNetwork{}, the graph has $7.41\times10^{11}$ topological orders but only around $2003$ equivalence classes on average\footnote{We mean the average value of $|\Lambda^x_G|$ over the set of features.}. Similar gaps are observed in balanced trees, where the difference can reach several orders of magnitude.
At the same time, the cost of computing equivalence classes remains moderate for graphs of practical size, typically under a few seconds. For example, in \childNetwork{}, all equivalence classes are computed in less than one second.
This reduction directly decreases the number of evaluations of the characteristic function, making the equivalence class approach substantially more efficient.

\vspace{-3mm}

\paragraph{Performance of Equivalence Class Algorithms} We now analyze the cost of computing equivalence classes.
As shown in Figure~\ref{fig:equivalenceClassesTimePlot}, the computation remains efficient for moderate graph sizes, typically under a few seconds. However, as graph size increases, the runtime grows alongside the number of equivalence classes.
We can see in Figure~\ref{fig:timeVsEquivClasses} that the runtime is strongly correlated with the number of equivalence classes. This is useful not only for explaining the observed growth, but also from a practical point of view. Since we also have a procedure to count the number of equivalence classes beforehand, we can use it as a rough predictor of the computational cost of the exact algorithm before actually running it.

\begin{figure}[ht]
	\centering
	
	\begin{subfigure}[t]{0.48\linewidth}
		\vspace{0pt}
		\centering
		\includegraphics[width=\linewidth]{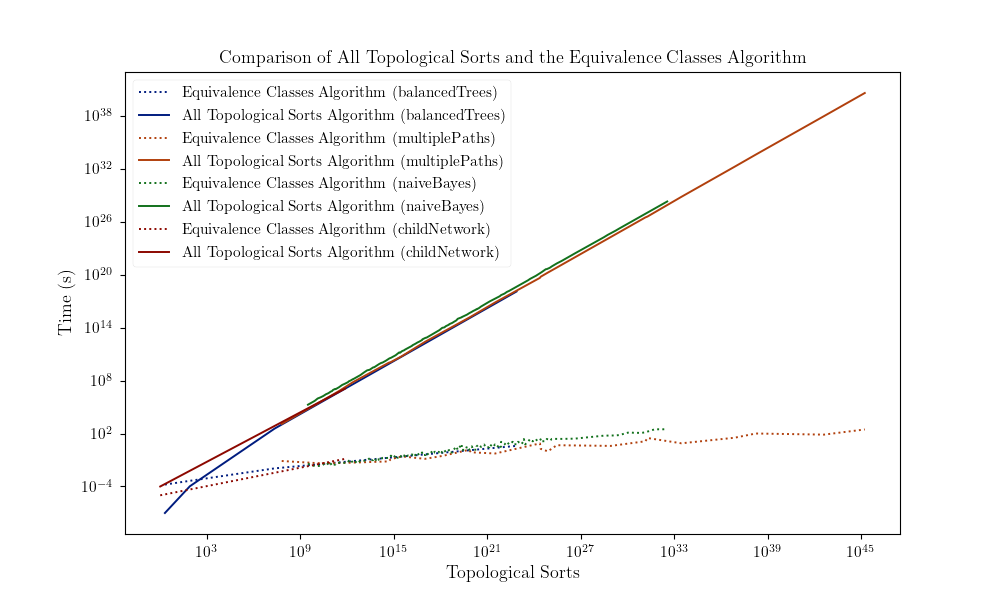}
		\caption{Runtime comparison between computing equivalence classes and enumerating all topological orders.\footnotemark}
		\label{fig:equivalenceClassesVsToposortsTimePlot}
	\end{subfigure}
	\hfill
	\begin{subfigure}[t]{0.48\linewidth}
		\vspace{0pt}
		\centering
		\includegraphics[width=\linewidth]{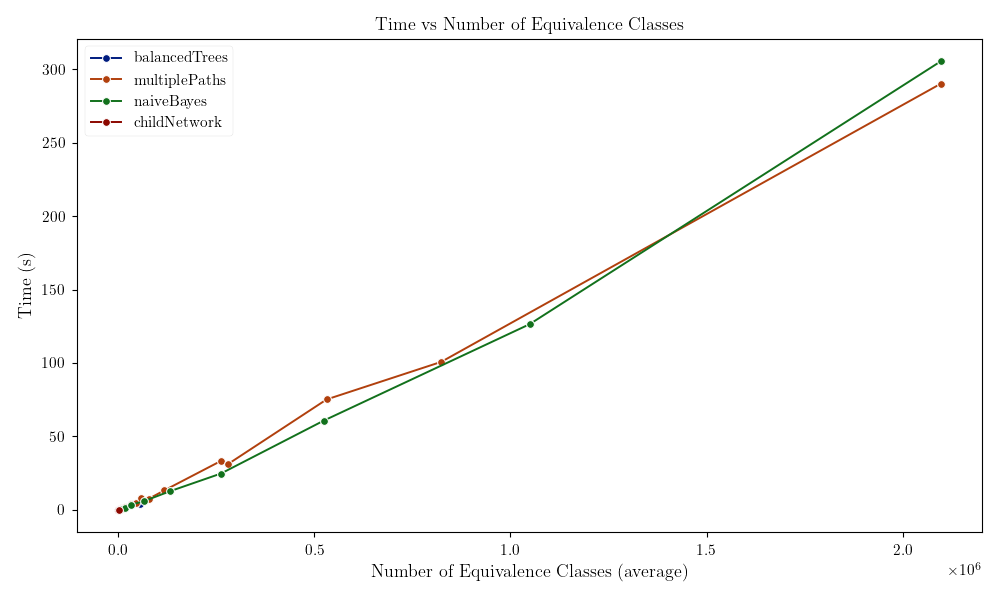}
		\caption{Runtime as a function of the number of equivalence classes.}
		\label{fig:timeVsEquivClasses}
	\end{subfigure}
	\label{fig:runtime_experiments}
\end{figure}
\footnotetext{To estimate runtime, we measured the time required to generate the first 1000 topological orders and extrapolated using the total number of orders in each graph, since full execution was not computationally feasible.}

\vspace{-0.4cm}

Figure~\ref{fig:equivalenceClassesVsToposortsTimePlot} highlights the main advantage of our approach. While enumerating all topological orders quickly becomes infeasible (e.g., requiring days for large graphs), computing equivalence classes remains tractable.

\vspace{-3mm}

\paragraph{Sampling-Based Approximation of Topological Orders} We now evaluate the sampling-based approach.

\begin{figure}[H]
	\centering
	
	\begin{subfigure}[t]{0.32\linewidth}
		\vspace{0pt}
		\centering
		\small
		\setlength{\tabcolsep}{4pt}
		\renewcommand{\arraystretch}{1.1}
		\begin{tabular}{|r|r|}
			\hline
			\textbf{\# Samples} & \textbf{Time (s)}\\
			\hline
			100   & 0.9270 \\
			1000  & 1.1170 \\
			10000 & 4.7170 \\
			20000 & 7.7170 \\
			30000 & 10.8387 \\
			\hline
		\end{tabular}
		\caption{Sampling times for topological orders in \childNetwork{} using our developed algorithm.}
		\label{fig:samplingTimesTopoSorts}
	\end{subfigure}
	\hfill
	\begin{subfigure}[t]{0.64\linewidth}
		\vspace{0pt}
		\centering
		\includegraphics[width=0.6\linewidth]{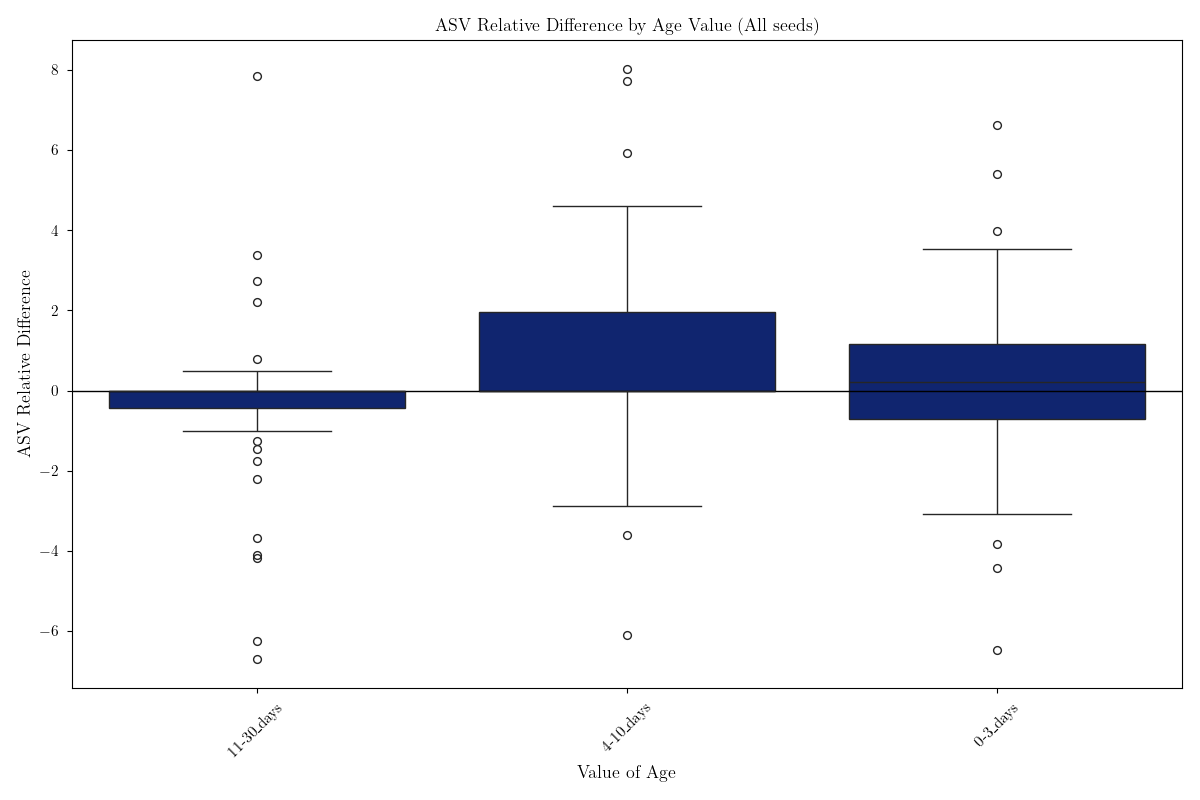}
		\caption{Relative error between approximate and exact ASV values in \childNetwork{}.}
		\label{fig:boxplotASVApproximateDifferences}
	\end{subfigure}
	\label{fig:sampling_results_combined}
\end{figure}

\vspace{-0.4cm}

Table~\ref{fig:samplingTimesTopoSorts} shows that when using our algorithm the sampling remains efficient even for large numbers of samples. Moreover, the runtime is not linear in the number of samples, as multiple candidates are generated simultaneously during recursive calls due to our efficient implementation.
We conclude that in the bounded-degree setting (observe that in all the examples considered here most nodes have small degree) our algorithm is more efficient than the alternative options such as \cite{HUBER2006420}. Nevertheless, we recall that their algorithm is more general, as it applies beyond polytrees and runs in polynomial time.

Finally, we evaluate the approximation error of ASV using sampled orders. Figure~\ref{fig:boxplotASVApproximateDifferences} shows that using 1000 samples yields an average relative error of around $8\%$. Larger ASV values exhibit smaller relative error, while larger deviations correspond to very small true values.
These results indicate that sampling provides a practical trade-off between accuracy and computational cost.

\vspace{-3mm}

\paragraph{Remarks on Mean Prediction Computation} In our experiments, the cost of evaluating the characteristic function is not the dominant factor. This happens because for decision trees under Bayesian Network distributions with a polytree structure mean prediction can be computed exactly in an efficient way. In general models this might not be the case, but nonetheless due to Proposition~\ref{prop:sampling_to_compute_ASV} if sampling is used instead of exact computation the overall error of the approximation should not be extremely bad.

\section{Conclusions and open questions}

In this work, we studied \textit{Asymmetric Shapley Values} (ASV), a variant of Shapley values that incorporates causal information into feature attribution. We analyzed the computational complexity of ASV and showed tractability results for restricted settings in which computing SHAP was intractable. To make ASV more practical, we developed exact and approximate algorithms based on equivalence classes of topological orderings. Our experiments show that these techniques improve the performance for computing the score and yield accurate approximations.

Overall, our results suggest that ASV is a promising explainability framework when causal structure is relevant, since it can capture dependencies that standard SHAP does not reflect. Our experiments also indicate that ASV can be implemented effectively for several networks, although its scalability depends strongly on the causal graph. While equivalence classes can greatly reduce the number of topological orderings considered, they may still grow quickly in large or dense DAGs, where the tree or polytree assumptions used by the exact algorithms no longer apply and inference can become more expensive. In these cases, the method may need to rely on sampling-based approximations, whose accuracy depends on both topological-order sampling and conditional-expectation estimation. Thus, further optimization is needed for very large or highly connected graphs.

Several directions remain open for future work. On the theoretical side, it would be interesting to generalize the equivalence class algorithm from directed trees to polytrees, and to further study the complexity of computing the ASV in other cases. On the practical side, possible improvements include implementing alternative counting algorithms from the literature, optimizing the current implementation through caching or parallelization, and extending the framework to more general causal models beyond Bayesian networks.

\begin{credits}

\subsubsection{\discintname}The authors have no competing interests to declare that are
relevant to the content of this article.
\end{credits}
\printbibliography

@article{shapley1953value,
  title={A value for n-person games},
  author={Shapley, Lloyd S and others},
  journal={Classics in game theory},
  year={1953},
  publisher={Princeton University Press Princeton}
}

@article{fryer2021shapley,
  title={{Shapley} values for feature selection: The good, the bad, and the axioms},
  author={Fryer, Daniel and Str{\"u}mke, Inga and Nguyen, Hien},
  journal={Ieee Access},
  volume={9},
  pages={144352--144360},
  year={2021},
  publisher={IEEE}
}

@article{huang2023inadequacy,
  title={The inadequacy of {Shapley} values for explainability},
  author={Huang, Xuanxiang and Marques-Silva, Joao},
  journal={arXiv preprint arXiv:2302.08160},
  year={2023}
}

@incollection{marques2023logic,
  title={Logic-based explainability in machine learning},
  author={Marques-Silva, Joao},
  booktitle={Reasoning Web. Causality, Explanations and Declarative Knowledge: 18th International Summer School 2022, Berlin, Germany, September 27--30, 2022, Tutorial Lectures},
  pages={24--104},
  year={2023},
  publisher={Springer}
}

@article{van2022tractability,
  title={On the tractability of {Shap} explanations},
  author={Van den Broeck, Guy and Lykov, Anton and Schleich, Maximilian and Suciu, Dan},
  journal={Journal of Artificial Intelligence Research},
  volume={74},
  pages={851--886},
  year={2022}
}

@article{arenas2023complexity,
  title={On the complexity of {Shap}-score-based explanations: Tractability via knowledge compilation and non-approximability results},
  author={Arenas, Marcelo and Barcel{\'o}, Pablo and Bertossi, Leopoldo and Monet, Mika{\"e}l},
  journal={Journal of Machine Learning Research},
  volume={24},
  number={63},
  pages={1--58},
  year={2023}
}

@article{frye2019asymmetric,
  title={Asymmetric {Shapley} values: incorporating causal knowledge into model-agnostic explainability},
  author={Frye, Christopher and Rowat, Colin and Feige, Ilya},
  journal={arXiv preprint arXiv:1910.06358},
  year={2019}
}

@article{pearl1986bayesianInference,
  title={Fusion, Propagation, and Structuring in Belief Networks},
  author={Pearl, Judea},
  journal={Artificial Intelligence},
  volume={29},
  number={3},
  pages={241-288},
  year={1986},
  publisher={Elsevier Science Publishers B.V. (North-Holland)}
}

@article{Cooper1990,
	author = {Gregory F. Cooper},
	title = {The Computational Complexity of Probabilistic Inference Using Bayesian Belief Networks},
	journal = {Artificial Intelligence},
	volume = {42},
	number = {2--3},
	pages = {393--405},
	year = {1990},
	doi = {10.1016/0004-3702(90)90060-D}
}

@misc{shapOriginalPaper,
      title={A Unified Approach to Interpreting Model Predictions}, 
      author={Scott Lundberg and Su-In Lee},
      year={2017},
      eprint={1705.07874},
      archivePrefix={arXiv},
      primaryClass={cs.AI},
      url={https://arxiv.org/abs/1705.07874}, 
}

@article{countingLinearExtensions,
  author    = {Graham R. Brightwell and Peter Winkler},
  title={Counting Linear Extensions},
  journal   = {Order},
  volume    = {8},
  pages     = {225--242},
  year      = {1991},
  doi       = {10.1007/BF00383444},
  url       = {https://link.springer.com/article/10.1007/BF00383444},
}

@article{MILLER20191,
title={Explanation in {Artificial intelligence}: Insights from the social sciences},
journal = {Artificial Intelligence},
volume = {267},
pages = {1-38},
year = {2019},
issn = {0004-3702},
doi = {https://doi.org/10.1016/j.artint.2018.07.007},
url = {https://www.sciencedirect.com/science/article/pii/S0004370218305988},
author = {Tim Miller}
}

@misc{lipton2017mythosmodelinterpretability,
      title={The Mythos of Model Interpretability}, 
      author={Zachary C. Lipton},
      year={2017},
      eprint={1606.03490},
      archivePrefix={arXiv},
      primaryClass={cs.LG},
      url={https://arxiv.org/abs/1606.03490}, 
}

@article{HUBER2006420,
title = {Fast perfect sampling from linear extensions},
journal = {Discrete Mathematics},
volume = {306},
number = {4},
pages = {420-428},
year = {2006},
issn = {0012-365X},
doi = {https://doi.org/10.1016/j.disc.2006.01.003},
url = {https://www.sciencedirect.com/science/article/pii/S0012365X06000033},
author = {Mark Huber}
}

@inproceedings{efficientCountingOfToposorts,
author = {Kangas, Kustaa and Hankala, Teemu and Niinim\"{a}ki, Teppo and Koivisto, Mikko},
title = {Counting linear extensions of sparse posets},
year = {2016},
isbn = {9781577357704},
publisher = {AAAI Press},
booktitle = {Proceedings of the Twenty-Fifth International Joint Conference on Artificial Intelligence},
pages = {603–609},
numpages = {7},
location = {New York, New York, USA},
series = {IJCAI'16}
}

@article{bubley1999faster,
	title={Faster random generation of linear extensions},
	author={Bubley, Russ and Dyer, Martin},
	journal={Discrete mathematics},
	volume={201},
	number={1-3},
	pages={81--88},
	year={1999},
	publisher={Elsevier}
}

@article{mersha2024explainable,
  title={Explainable artificial intelligence: A survey of needs, techniques, applications, and future direction},
  author={Mersha, Melkamu and Lam, Khang and Wood, Joseph and Alshami, Ali K and Kalita, Jugal},
  journal={Neurocomputing},
  volume={599},
  pages={128111},
  year={2024},
  publisher={Elsevier}
}

\appendix

\section{Appendix}

\begin{proof}[of Theorem~\ref{the:asv_naive_bayes_tractable_clean}]
	First, we show the right-to-left implication, which is the hard one. Let $x_1$ be the root node of the causal graph, and let $x_2,\ldots, x_n$ be the rest of the nodes. The DAG has $(n-1)!$ topological sorts, and in particular it holds that, for any $2 \leq j \leq n$,\footnote{In the proofs we write the characteristic function as $\nu_{M,e,\Pr}$ to make the dependence on the distribution $\Pr$ explicit.}
	\begin{align*}
	\phi^{ASV}_{M,e,\Pr}(x_j) = \frac{1}{(n-1)!} \sum_{\pi \in \perm(\{x_2, \ldots, x_n\})} [ \charactheristicFunction_{M,e,\Pr}(\{x_1\} &\cup \pi_{<j} \cup \{x_j\}) \\
	&- \charactheristicFunction_{M,e,\Pr}(\{x_1\} \cup \pi_{<j}) ]
	\end{align*}
	Note that once $x_1$ is fixed the distribution over $x_2,\ldots, x_n$ is a product distribution with probabilities $\Pr\left( X_j = 1 | X_1 = e(x_1) \right)$ for $2 \leq i \leq n$.
	For simplicity, let's assume $e(x_1) = 1$. Then, if we define the product distribution $\Pr'$ given by
	\[
	\Pr\!'[X_i = 1] = 
	\begin{cases}
		1 & i = 1 \\
		\Pr[X_i = 1 | X_1 = 1] & \text{otherwise}
	\end{cases}
	\]
	it is easy to see that for any $S \subseteq X \setminus \{x_1\}$ it holds that
	\begin{align}\label{eq:prob_prime_for_fixing_x1}
		\nu_{M,e,\Pr}(\{x_1\} \cup S) = \nu_{M,e,\Pr'}(S) 
	\end{align}
	This is intuitive: the distribution $\Pr'$ is obtained by fixing the value of $x_1$ as $e(x_1)$, and thus having $x_1$ inside $\nu$ is equivalent to using the distribution $\Pr'$.
	
	From Eq.~\eqref{eq:prob_prime_for_fixing_x1} it can be proven easily that
	\begin{align}
		\phi^{ASV}_{M,e,\Pr} (x_j) = \phi^{Shap}_{M,e,\Pr'}(x_j)
	\end{align}
	which implies that we effectively reduced the problem of computing the asymmetric Shapley values for any feature $x_j$ with $j\geq 2$ to computing the Shapley values for some product distribution. Finally, for $x_1$ it can be seen that
	\begin{align*}
		\phi^{ASV}_{M,e,\Pr}(x_1) = \nu_{M,e,\Pr{}}(\{x_1\}) - \nu_{M,e,\Pr{}}(\emptyset)
	\end{align*}
	because $\pi_{<1} = \emptyset$ for any permutation consistent with the DAG. Computing the values $\nu_{M,e,\Pr{}}(\{x_1\})$ and $\nu_{M,e,\Pr{}}(\emptyset)$ reduces to computing averages of the model $M$ with some product distributions, and it is known that the capacity of computing 
	the Shapley values for any product distribution entails the capacity of computing averages \parencite{van2022tractability}.
	
	The left-to-right implication is immediate by observing that product distributions are a particular case of Naive Bayes distribution, and thus we can ``simulate'' them: one possible way to do this is
	to consider, for each $1 \leq i \leq n$, the Naive Bayes distribution $\Pr_i$ given by placing $x_i$ as parent node with no conditioning on its children (i.e. $\Pr_i\left[x_j=1 | x_i = 0\right] = \Pr_i\left[x_j=1 | x_i = 1\right]$). Then, it holds that
	\begin{align*}
		\phi^{Shap}_{M,e,\Pr}(x_j) = \frac{1}{n} \sum_{i=1}^n \phi^{ASV}_{M,e,\Pr_i}(x_j)
	\end{align*}
\end{proof}

\begin{proof}[of Proposition~\ref{prop:bound}]
	For every unrelated feature $y$ (in the sense that $y$ is neither a descendant nor ancestor of $x$) let $C_y \subseteq C$ be the directed tree rooted at $y$ induced by the nodes reachable by $y$. Let $f(y)$ denote the number of ways that the features in $C_y$ can be arranged in the permutations, i.e. 
	\begin{align}\label{eq:f}
		f(y) = \big| \{S \subseteq C_y : \exists \pi \in topos(C) \text{ such that } \pi_{<x} \cap C_y = S\} \big|.
	\end{align}
	Note that
	\begin{align}\label{eq:fun_for_equi_classes}
		f(y) = \begin{cases}
			2 & y \text{ is a leaf}\\
			1 + \prod_{z \text{ child of } y} f(z) & \text{otherwise}
		\end{cases}
	\end{align}
	In the base case the set of Eq.~\eqref{eq:fun_for_equi_classes} is $\{\emptyset, \{y\}\}$. For the recursive case the reasoning is the following: either we do not put $y$ in the set (and then all other features reachable from $y$ won't be in the set as well) or rather 
	we put $y$ and have complete freedom to decide which features to put from the subtrees of $y$.

	We prove by induction that $f(y) \leq h_y^{l_y} + 1$ where $l_y$ is the number of leaves of $C_y$ and $h_y$ its depth. This clearly holds whenever $y$ is a leaf itself (in that case, we take $h_y = 1$). For the recursive case, note that
	\begin{align}\label{eq:bounding}
		f(y) &= 1 + \prod_{z \text{ child of } y} f(z)\\
		&\leq 1 + \prod_{z \text{ child of } y} (h_z^{l_z} + 1)\nonumber\\
		&\leq 1 + \prod_{z \text{ child of } y} (h_z+1)^{l_z}\nonumber\\
		&\leq 1 + \prod_{z \text{ child of } y} h_y^{l_z}\nonumber\\
		&= 1 + h_y^{\sum_{z \text{ child of } y} l_z} = 1 + h_y^{l_y}\nonumber
	\end{align}

	Now, let $u_1,\ldots,u_k$ be nodes in the causal graph which are not ancestors of $x$ but are immediate descendants of some proper ancestor of $x$ (look Figure~\ref{fig:ASV_forest_example} for an example). It follows by the definition of $f$ that
	\begin{align}\label{eq:Lambda}
		|\Lambda_C^x| = \prod_{i=1}^k f(u_i)
	\end{align}
	and therefore
	\begin{align*}
		|\Lambda_C^x| \leq h^l + 1
	\end{align*}
	where the inequality is obtained using arguments similar to those from Eq.~\eqref{eq:bounding}.
\end{proof}

\begin{proof}[of Proposition~\ref{prop:algo}]
	The algorithm implements a recursion inspired by Eq.~\eqref{eq:fun_for_equi_classes}. First we recursively compute the equivalence classes of each unrelated subtree. We then combine these classes with one another, and finally integrate them with the ancestors and descendants while respecting the precedence constraints imposed by the DAG. In particular, descendants are added on the right of $x_i$, while including ancestors requires counting the valid interleavings with the nodes placed to its left. This yields all equivalence classes and their sizes without explicitly enumerating topological orders.

	To be more precise, let $u_1,\ldots,u_k$ be the nodes that are not ancestors of $x_i$ and are immediate descendants of some ancestor of $x_i$ (look Figure~\ref{fig:ASV_forest_example} for an example). Let $B_{j}$ be, for each $u_j$, the set inside the definition of $f(u_j)$ from Eq.~\eqref{eq:f}. What we need to count is, for every way of choosing $S_j \in B_{j}$, the number of topological sorts such that $\pi_{<x} = A \cup \bigcup_{j=1}^k S_j$ where $A$ are the ancestors of $x_i$.
	This computation is not straightforward because, even though there are no immediate constraints in the ordering of the sets $B_j$, nodes from $B_j$ cannot appear in the topological sort before the common ancestor of $u_j$ and $x_i$. Nonetheless, these values can be computed in a recursive manner using standard counting techniques. Moreover, when implemented using dynamic programming such a recursion has a complexity that grows polynomially on the number of nodes and $\prod_{j=1}^k f(u_j)$, which equals $|\Lambda^x_C|$ (see Eq.~\eqref{eq:Lambda}).
\end{proof}

\begin{proof}[of Proposition~\ref{prop:sampling_to_compute_ASV}]
	Consider the random variable $\randomVar$ defined over the uniform distribution $\mathcal{D}_{unif}$ over $topos(C)$ given by 
	\[
	\randomVar(\pi) = \left[ \nu_{M,e}(\pi_{<x} \cup \{x\}) - \nu_{M,e}(\pi_{<x}) \right].
	\]
	It holds that
	\begin{align*}
		\mathbb{E}[\randomVar] = \phi^{ASV}_{M,e}(x)
	\end{align*}
	Now, if $\randomVar_{\varepsilon_{samp}}$ denotes the analogous random variable distributed according to the distribution $\mathcal{D}_{\nu}$ induced by the sampling procedure $\procedureSampling$,
	\begin{align*}
		|&\mathbb{E}[\randomVar] - \mathbb{E}[\randomVar_{\varepsilon_{samp}}]| = \\
		&= \left|\sum_{\pi \in topos(C)} \left[\Pr\left( \mathcal{D}_{unif} = \pi \right) - \Pr\left( \mathcal{D}_{\nu}= \pi \right) \right] \left[ \nu_{M,e}(\pi_{<x} \cup \{x\}) - \nu_{M,e}(\pi_{<x}) \right] \right|\\
		&\leq \sum_{\pi \in topos(C)} \left|\Pr\left( \mathcal{D}_{unif} = \pi \right) - \Pr\left( \mathcal{D}_{\nu} = \pi \right) \right| \leq \varepsilon_{samp}
	\end{align*}
	Then, to approximate $\mathbb{E}[\randomVar]$ we can simply estimate $\mathbb{E}[\randomVar_{\varepsilon_{samp}}]$. To do so we apply the procedure $\procedureSampling$ to compute a subset $\subsetPerm$ of the topological orderings and then use $\procedureNu$ to obtain
	\begin{align*}
		\mathbb{E}[\randomVar_{\varepsilon_{samp}}] \sim \frac{1}{|\subsetPerm|} \sum_{\pi \in \subsetPerm} \left[ \procedureNu(\pi_{<x} \cup \{x\}) - \procedureNu(\pi_{<x}) \right]
	\end{align*}
	Using Hoeffding's inequality we can ensure that by taking $O(1/\varepsilon_{\nu}^2)$ samples the set $\subsetPerm$ is a good representation of the average with high probability, i.e.
	\begin{align*}
		\left|\mathbb{E}[\randomVar_{\varepsilon_{samp}}] - \frac{1}{|\subsetPerm|} \sum_{\pi \in \subsetPerm} \left[ \nu_{M,e}(\pi_{<x} \cup \{x\}) - \nu_{M,e}(\pi_{<x}) \right]\right| \leq \varepsilon_\nu
	\end{align*}
	Finally, after sampling $K$ we merely approximate each call to $\nu_{M,e}$ with $\procedureNu$. It holds that
	\begin{align*}
		\Bigg| \frac{1}{|\subsetPerm|} &\sum_{\pi \in \subsetPerm} \left[ \nu_{M,e}(\pi_{<x} \cup \{x\}) - \nu_{M,e}(\pi_{<x}) \right] - \frac{1}{|\subsetPerm|} \sum_{\pi \in \subsetPerm} \left[ \procedureNu(\pi_{<x} \cup \{x\}) - \procedureNu(\pi_{<x}) \right]\Bigg|\\
		&\leq \frac{1}{|\subsetPerm|} \sum_{\pi \in \subsetPerm} \left[ |\nu_{M,e}(\pi_{<x} \cup \{x\}) - \procedureNu(\pi_{<x} \cup \{x\})| + |\nu_{M,e}(\pi_{<x}) - \procedureNu(\pi_{<x})| \right]\\
		&\leq 2 \varepsilon_{\nu}
	\end{align*}
	The result from the statement is obtained by applying the triangle inequality with all the obtained equations.
\end{proof}

\end{document}